\documentclass[10pt, a4paper]{article}
\usepackage[final]{lrec2026}
\usepackage{times}
\usepackage[utf8]{inputenc}
\usepackage[T1]{fontenc}
\usepackage{graphicx}
\usepackage{booktabs}
\usepackage{amsmath}
\usepackage{amsthm}
\usepackage{amssymb}
\usepackage{tabularx}
\usepackage{svg}
\usepackage{multirow}
\usepackage{soul}
\usepackage{algorithm}
\usepackage{algpseudocode}
\newtheorem{lemma}{Lemma}
\newtheorem{theorem}{Theorem}
\newtheorem{definition}{Definition}

\usepackage{microtype}
\usepackage[mode=buildnew]{standalone}
\usepackage{tikz,pgf}
\usetikzlibrary{matrix, fit, arrows.meta,shapes.geometric, arrows,shadows,decorations.pathreplacing,fit, positioning}

\newcommand{\mbf}[1]{\textbf{#1}}
\newcommand{\cont}{\mathrm{Cont}}

\title{LLM-based Atomic Propositions Help Weak Extractors:\\ Evaluation of a Propositioner for Triplet Extraction}

\name{Luc Pommeret\textsuperscript{1}, Thomas Gerald\textsuperscript{1}, Patrick Paroubek\textsuperscript{1}, \\ \large\textbf{Sahar Ghannay\textsuperscript{1}, Christophe Servan\textsuperscript{2}, Sophie Rosset\textsuperscript{1}}}

\address{$^{1}$Université Paris-Saclay, CNRS, LISN, 91400, Orsay, France \\$^{2}$ AMIAD, Pôle Recherche, 91120, Palaiseau, France \\
         surname.name@lisn.fr}
\abstract{
Knowledge Graph construction from natural language requires extracting structured triplets from complex, information-dense sentences. In this paper, we investigate if the decomposition of text into \textit{atomic propositions} (minimal, semantically autonomous units of information) can improve the triplet extraction.
We introduce \texttt{MPropositionneur-V2}, a small multilingual model covering six European languages trained by knowledge distillation from \texttt{Qwen3-32B} into a \texttt{Qwen3-0.6B} architecture, and we evaluate its integration into two extraction paradigms: entity-centric (\texttt{GLiREL}) and generative (\texttt{Qwen3}). Experiments on SMiLER, FewRel, DocRED and CaRB show that atomic propositions benefit weaker extractors (\texttt{GLiREL}, \texttt{CoreNLP}, 0.6B models), improving relation recall and, in the multilingual setting, overall accuracy. For stronger LLMs, a fallback combination strategy recovers entity recall losses while preserving the gains in relation extraction. These results show that atomic propositions are an interpretable intermediate data structure that complements extractors without replacing them.
\\
\newline 
}

\begin{document}
\maketitleabstract

\begin{figure*}[t]
    \centering
    \includegraphics[width=\textwidth]{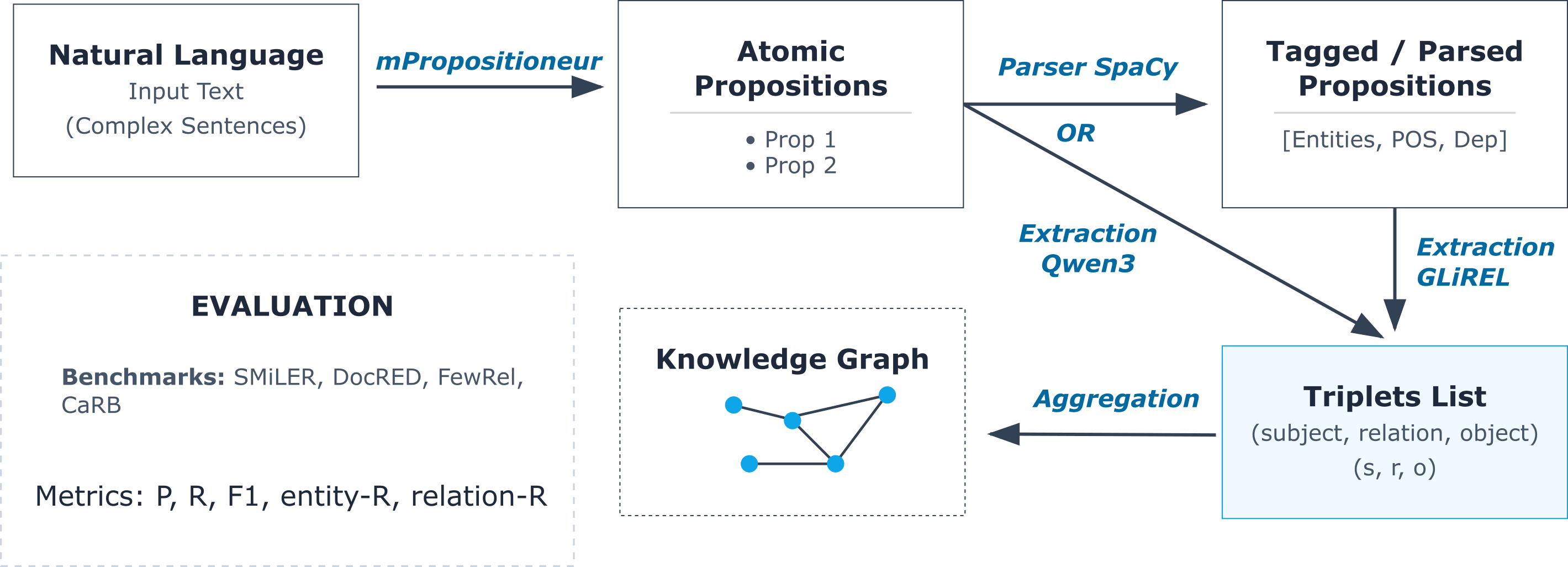}
    \caption{The upper schema depicts the pipeline's stages: at stage 1, we extract atomic propositions from the source text; at stage 2, we extract triplets using either a dependency parser or generative LLMs; and finally, we build the knowledge graph from the entities and relations retrieved. For evaluation, we limit to the triplet entity-relation benchmark.}
    \label{fig:pipeline}
\end{figure*}

\section{Introduction}


The interpretability of Natural Language Processing (NLP) models is a requirement for applications like fact-checking and automated construction of Knowledge Graphs (KGs). While neural models have achieved state-of-the-art results, their internal mechanisms remain opaque. Most current explainability methods are \textit{post hoc}, seeking to explain decisions after the training. 

In this paper, we argue for a shift towards interpretability by design, where the data structure itself is traceable. Our approach is based on 
natural language (NL)
\textit{atomic propositions}, a minimal, semantically autonomous unit of information. 
For defining and producing NL atomic propositions, we rely on the formalism of Semantic Information Theory\cite{CarnapBarHillel}, which defines the decomposition of non-atomic propositions into atomic ones \footnote{In logic, atomic propositions are closed formulae that cannot be further decomposed into several smaller propositions without loss of semantic integrity or introduction of "bad" informational cuts (see Section~\ref{sec:formalism})}.
For instance,  the sentence "The cat and the dog are in the kitchen" atomizes into two NL atomic propositions: "The cat is in the kitchen" and "The dog is in the kitchen". Neither of these two propositions can be decomposed further, because otherwise we would add to the original semantic content hallucinations (spurious information). If the original proposition was not a conjunction but a disjunction ("The cat \textbf{or} the dog'') decomposing it would introduce new information since we would have a disjunction of two atomic propositions each one holding more semantic information than the original text.
More details are available in section~\ref{sec:formalism}. 

For the implementation of the sentence decomposition process into NL atomic propositions, we rely in this first experiment on a limited-depth recursive 
prompting of a multilingual LLM, an
atomic proposition being defined as a prompting fixed-point.
To perform the atomisation, we propose a small multilingual model, the \texttt{MPropositionneur-V2}, that can perform coreference resolution to produce atomic propositions from text.
For more details, please see section~\ref{sec:pipeline}.


We aim to show that in Open Information Extraction (OpenIE), using NL atomic propositions can improve performance while preserving interpretability. Here, the interpretability is the auditability of the data structure, i.e. the NL atomic propositions.
These propositions serve as an intermediary representation in the process of mapping a natural language sentence $S$ to a set of formal triplets $\{ (s, r, o) \}$, where $s$ is the subject entity, $o$ the object entity, and $r$ the relation.
In our experiment, we evaluate whether splitting textual information into atomic propositions could positively impact triplet extraction. We hypothesize that extracting entities and relations could be made easier from less information-dense text chunks, especially NL atomic propositions. 

\textbf{Our contributions} consist of a new multilingual propositioner and a pipeline leveraging atomic propositions to enhance entity-relation extraction based on LLM.

We report the evaluation of the multilingual propositioner trained through distillation of a large language model (LLM), compared to a baseline composed of a rule-based method leveraging a dependency parser.
To assess the effectiveness of our model and the extraction pipeline, we evaluate the method on different entity-relation benchmarks for relation extraction and triplet validation.
These benchmarks cover a variety of tasks, such as open and closed information extraction, sentence- and document-based extraction, and extraction based solely on relations and triplets.


The paper is organised as follows: First, in Section  \ref{sec:sota}, we review related works, describing previous triplet extraction approaches and methods based on propositioners. Section \ref{sec:formalism}  presents the formalism of the atomic proposition. Section \ref{sec:pipeline}  describes our pipeline in depth. We then present the experimental protocol to answer the research question in Section \ref{sec:protocol}. Section \ref{sec:results}  discusses the different results, and Section \ref{sec:conclusion} concludes with suggestions for future work.

\section{Related Works}
\label{sec:sota}

For inference and/or information retrieval \cite{xiang2026usegraphsragcomprehensive}, it is common to represent information as a Knowledge Graph (KG). While a wide range of KGs automatically extracted from many different sources is available, for instance, using wikidata taxonomy and entities \cite{waagmeester2020wikidata,hassanzadeh2021building}, extracting KGs from natural language source of information remains a significant challenge \cite{waagmeester2020wikidata}. 
We can split the wide range of approaches for extracting automatically KG into two categories.

The first one performs named entity recognition (NER), and then applies entity linking and relation extraction approaches to build a knowledge graph.
For example, in the sentence 'The Eiffel Tower is in Paris', the entities must be linked to the knowledge graph (Paris is entry Q90 in Wikidata) and the relation must be linked to a vocabulary of relations (here, hasLocation, property P131 in Wikidata) \cite{ClosedIE}. 
However, this approach lacks flexibility when dealing with the intrinsic complexity of natural language. 

The second kind of approach is called Open Information Extraction (OpenIE), and it aims to extract triplets $(s,r,o)$\footnote{(subject, relation, object)}, without relying on predefined sets of entities and/or relations \cite{OpenIE}.

To extract triplets from sentences, modern approaches use LLMs, especially transformer encoder or decoder models. 
Models like \texttt{mREBEL} \citep{mrebel} generate triplets directly from complex input sentences. 
Such generative models often struggle with long-range dependencies, nested clauses, and coordination, leading to lower recall on complex structures.
More recently, \texttt{GLiREL} \citep{glirel} performs zero-shot relation classification by jointly encoding entity pairs and candidate relation labels.
It is more robust but can be sensitive to the quality of the initial entity recognition.
One way to address this duality is through the recourse of sentence simplification for relation extraction.




Simplifying text before extracting relations has been explored, and, for instance, \cite{miwa} proposed an entity-focused sentence simplification method to improve relation extraction.
More recently, \cite{niklaus} introduced a rule-based sentence simplification system that rewrites complex sentences into simpler sentences for Open Information Extraction. 
However, these approaches rely on syntactic rules or dependency parsing, are limited to a single language, and lack coreference resolution. 

Simplification by structural atomisation of information is of interest for many tasks: in Retrieval Augmented Generation (RAG), the indexing of atomic propositions reduces the noise in dense retrieval~\cite{Chen2024}; in Natural Language Inference (NLI), context augmentation by atomic propositions increases performance~\cite{Stacey2023}; in fact-checking, it is possible to decompose the text and verify each atomic proposition recursively~\cite{Min2023}; and finally, in Summary Evaluation, this method allows one to approach human judgement~\cite{DBLP:conf/coria/HerserantG25}. 

Recently, \cite{Min2023} proposed the use of LLM-based atomic propositions in NLP.
The atomic propositions approach combines cutting the information into small sentence pieces (the smallest we can, without loss of information, see Section~\ref{sec:formalism}) and coreference resolution.
One advantage of atomic propositions is the \textit{structure}. 
This approach gives a fixed structure to the information we want to verify, which is more easily parseable and has been used for Information Retrieval \cite{Chen2024} and Summary Evaluation \cite{DBLP:conf/coria/HerserantG25}. 





Based on these recent works, we propose using atomic propositions to construct knowledge graphs.   
We propose using a propositioner based on a distilled multilingual language model to perform the atomic propositions extraction. This model handles simplification and coreference resolution across six languages and is grounded in a formal framework.
Once the propositions are extracted, the text input is flattened. Then, we experiment with different methods to extract entity-relation triplets. For each triplet, we produce a graph that links the entities with the relations produced by the model.


\section{Formal Framework}
\label{sec:formalism}

The abstraction underlying our objectives is the theoretical atomic proposition.
To enlighten the atomization process, we use the formalism of Semantic Information Theory \cite{CarnapBarHillel}. 
This formalism provides a strong understanding of information in terms of signification and gives a criterion for cutting a proposition in a way that preserves information.

\subsection{Information Content}
\label{InfoContent}
Let $\phi$ be a formula. We define $I(\phi) = -\log_2(\frac{|\cont(\phi)|}{|W|})$, where $\cont(\phi)$ is the set of worlds\footnote{A world $W$ is a determined assignment of truth values for each atomic subformula of $\phi$.} satisfying $\phi$.
A cut of $\phi$ into $\psi$ is \textbf{safe} if $\psi$ is a sub-formula of $\phi$ and $I(\phi) > I(\psi)$. It is \textbf{bad} if $I(\phi) \leq I(\psi)$.

\subsection{The CNF Condition}

We prove (in Annex \ref{sec:proofs}) that a proposition is atomic if and only if it is a clause in a Conjunctive Normal Form (CNF).
\begin{itemize}
    \item \textbf{Conjunctions ($A \land B$)}: Splitting into $A$ or $B$ is safe ($I(A \land B) > I(A)$).
    \item \textbf{Disjunctions ($A \lor B$)}: Splitting into $A$ is bad ($I(A) > I(A \lor B)$), as it "hallucinates" specificity.
\end{itemize}
Thus, our propositioner is designed to split conjunctions while preserving the integrity of disjunctive facts.

\subsection{BCP Paradox}

The BCP paradox states that a contradiction (like $A \land \neg A$) has infinite information in the extended reals $\overline{\mathbb{R}} = \mathbb{R} \cup \{-\infty, +\infty\}$:

\begin{align*}
    -\log_2(\frac{|\cont(A \land \neg A)|}{|W|}) = -\log_2(0) = +\infty
\end{align*}

However, for clauses of a CNF, this paradox cannot arise because they cannot be contradictions \cite{logique}. Therefore, atomic propositions cannot have infinite information.

In the current implementation, we expect the model to follow the previous formal framework. However, we have only performed a manual and partial evaluation. If this is positive, we cannot guarantee that this will be the case for all extracted propositions. Additional experiments will be conducted outside the scope of this work.

\section{Proposed Approach}
\label{sec:pipeline}

The aim of this paper is to show that triplet extraction could benefit from atomic propositions.
We define different stages to extract triplets, with a full pipeline depicted in Figure~\ref{fig:pipeline}. The global pipeline consists of three stages:

\begin{enumerate}
    \item \textbf{Atomization}: The complex text is processed by \texttt{MPropositionneur-V2}. This model, distilled from \texttt{Qwen3-32B} into a \texttt{Qwen3-0.6B} architecture, recursively splits the text until each proposition is stable and autonomous by using the prompt proposed in Figure \ref{fig:PromptTemplateStage1}.
    \item \textbf{LLM prompting}: We use the \texttt{Qwen3-4B} model to generate triplets directly from an atomized chunk using a dedicated prompt (Figure \ref{fig:PromptTemplateStage2})
    \item \textbf{KG Building}: Extracted triplets are aggregated into a Knowledge Graph, where nodes represent entities and edges represent relations.
\end{enumerate}

As a baseline, we replace stage 2 with two sub-stages by using the \textbf{Parsing} and \textbf{Triplet Extraction} as follows:

\begin{itemize}
    \item[2.1.] \textbf{Parsing}: Each atomic proposition is parsed (here using \texttt{SpaCy} or \texttt{Stanza}) to extract part-of-speech (POS) tags, named entity tags, and dependency trees.
    \item[2.2.] \textbf{Triplet Extraction}: A large language model or a neural based pattern matcher (for example \texttt{GLiREL}) extracts triplet candidates $(s, r, o)$ from the simplified, parsed atoms.
\end{itemize}

\begin{figure}[h]
    \centering
    \begin{tikzpicture}[
    font=\ttfamily,
    every node/.style={draw=black!40,rounded corners, align=left, text width=8cm,inner sep=5pt},
    optional/.style={dashed, fill=blue!15},
    mandatory/.style={fill=orange!15},
    iins/.style={fill=green!15}
    ]

    
    \node[mandatory] (system) {
    You are an expert in disambiguation  \\
and information extraction. \\

    };
    
    \node[iins, below=0.1cm of system] (instruction) {
        You must decompose the text into  \\
atomic propositions (single facts)  \\
that are FULLY AUTONOMOUS. \\

ABSOLUTE RULES: \\
1. ZERO PRONOUNS: "He", "She", "They",  \\
   "His", "Her", "Its", "This one"  \\
   ARE FORBIDDEN. \\
   ALWAYS replace them with the  \\
   full name of the entity. \\
2. CONTEXT: Each sentence must be  \\
   readable alone without knowing  \\
   its source. \\
3. REPETITION: Repeat the subject  \\
   in EACH sentence. \\
    };



    \node[fill=black!10,below=0.1cm of instruction] (input) {
OUTPUT FORMAT: Only a JSON array of \\
strings. \\
\vspace{0.10in}
Title: {title} \\
Content: {content} \\
Output: \\
    };
\end{tikzpicture}
    \caption{Prompt Template used for the distillation of the propositioner used in stage 1.}
   \label{fig:PromptTemplateStage1}
   \vspace{-0.2cm}

\end{figure}

Figure~\ref{fig:Example} illustrates the input and the output of the entire pipeline.


\begin{figure}[h]
    \centering
    \begin{tikzpicture}[
    font=\ttfamily,
    every node/.style={draw=black!40,rounded corners, align=left, text width=8cm,inner sep=5pt},
    optional/.style={dashed, fill=blue!15},
    mandatory/.style={fill=orange!15},
    iins/.style={fill=green!15}
    ]

    

    
    \node[iins] (instruction) {
    Extract all factual  \\
(subject, predicate, object) \\
triples from the sentence. \\
One triple per line in the format: \\
subject | predicate | object \\
No explanations. If no triple can be  \\
extracted, write nothing. \\
    };


    \node[fill=black!10,below=0.1cm of instruction] (input) {
        Sentence: {text}\\
    };
\end{tikzpicture}
    \caption{Prompt Template used for the stage 2.}
   \label{fig:PromptTemplateStage2}
   \vspace{-0.2cm}

\end{figure}

\begin{figure}[h]
    \centering
    \begin{tikzpicture}[
    font=\ttfamily,
    every node/.style={draw=black!40,rounded corners, align=left, text width=8cm,inner sep=5pt},
    optional/.style={dashed, fill=blue!15},
    mandatory/.style={fill=orange!15},
    iins/.style={fill=green!15}
    ]

    
    \node[fill=black!10,below=0.1cm] (input) {
           Input: \texttt{"Marie Curie, a Polish-born physicist, won the Nobel Prize in Physics."}\\
    };
    \node[optional, below=0.1cm of input] (atomic) {
        Atomic Props: \texttt{["Marie Curie is a physicist.", "Marie Curie was born in Poland.", "Marie Curie won the Nobel Prize in Physics."]}\\
    };
        \node[optional, below=0.1cm of atomic] (tripled) {
        Parsed Triplets: \texttt{(Marie Curie, occupation, physicist)}, \texttt{(Marie Curie, birthplace, Poland)}, \texttt{(Marie Curie, award, Nobel Prize in Physics)}.\\
    };
\end{tikzpicture}
    \caption{Example of a sentence input processed through the whole pipeline, the atomic output, and the triplets extracted.}
   \label{fig:Example}
   \vspace{-0.2cm}

\end{figure}

\section{Experimental Protocol}
\label{sec:protocol}

\subsection{Propositioner}

We train a propositioner\footnote{Available here : \url{https://huggingface.co/Zual/MPropositionneur-V2}}, \textit{i.e.} a model that transforms a text input into a list of atomic propositions, via knowledge distillation \cite{distillation}, using \texttt{Qwen3-32B} as the teacher model and \texttt{Qwen3-0.6B} as the student. 
The training data consist of chunks of Wikipedia articles in six European languages: English, French, Spanish, Italian, German, and Portuguese \cite{wikipedia}. We trained the models for 2 epochs, on an A6000 NVIDIA GPU. 

\subsection{Nat. Lang. Recursive Propositioner}
In Algorithm~\ref{alg:propositioner}, we describe the recursive propositioner method. This algorithm is designed to recursively apply the \texttt{MPropositionneur-V2} to all natural language propositions generated ($\mathcal{P}_{i}$) until either all have been proved to be a propositioner fixed-point or the recursion depth has reached an empirical threshold value (here $N=5$). In the latter case, we return only the subset of proved atomic propositions ($\mathcal{A}_{i}$). 
\begin{algorithm}
\caption{propositioner}
\label{alg:propositioner}
\begin{algorithmic}[1]
\Require $N=5$, $t$ is a text, $i \in \mathbb{N}$, $\mathcal{M}$ is the propositioner,
    \State $i \leftarrow 0$;   $\mathcal{P}_0 \leftarrow  \mathcal{M}(t)$; \\
$\mathcal{A}_0 \leftarrow  \{ p \in \mathcal{P}_{0},\hspace{0.5em} \{p\} = \mathcal{M}(p)\}$
    \While {($\mathcal{P}_i \neq \mathcal{A}_i) \wedge (i < N)$}
    \State $\mathcal{P}_{i+1} \leftarrow \bigcup_{x \in \mathcal{P}_i} \mathcal{M}(x)$
    \State $\mathcal{A}_{i+1} \leftarrow \{ x \in \mathcal{P}_{i+1},\hspace{0.5em} \{ x \} = \mathcal{M}(x)\}$ 
    \State $i \leftarrow i + 1$
    \EndWhile
\State $\textbf{return}$ $\mathcal{A}_i$
\end{algorithmic}
\end{algorithm}

\subsection{Datasets and Benchmark}
We evaluate our pipeline on the SMiLER dataset \cite{smiler} (a multi-domain relation extraction benchmark covering 14 languages and various relation types) where the entities are already known and the task is to find the relations between entity pairs, on FewRel dataset \cite{fewrel}, on DocRED \cite{docred} (a Document-based triplet extraction dataset) and on CaRB \cite{carb} (an openIE dataset for triplet extraction).

\subsection{Metrics}
To evaluate triplet extraction performances we will consider the following metrics:
\begin{itemize}
    \item Precision (P), Recall (R), and F1-Score. Area Under the Curve (AUC) is also used for the CaRB benchmark natively.
    \item Entity Recall: The percentage of gold entities found in the output (exact match with substring and macrostring accepted).
    \item Relation Recall: The accuracy of extracted triplets vs. gold standards, by mapping to the finite vocabulary of \texttt{GLiREL} using a semantic mapper (cosine similarity with BERT, threshold optimised on the dev set).
\end{itemize}

\begin{table*}[!ht]
\centering
\small
\begin{tabular}{@{}llccccccccc@{}}
\toprule
& & \multicolumn{3}{c}{GLiREL} & \multicolumn{3}{c}{Qwen3-0.6B} & \multicolumn{3}{c}{Qwen3-4B} \\
\cmidrule(lr){3-5}\cmidrule(lr){6-8}\cmidrule(lr){9-11}
Benchmark & Pipeline & Acc & e-rec & r-rec & Acc & e-rec & r-rec & Acc & e-rec & r-rec \\
\midrule
\multirow{14}{*}{\rotatebox{90}{SMiLER}}
  & English -- Direct   & 49.8 & 99.0  & 50.3 & 35.7 & 100.0 & 35.7 & 71.4 & 100.0 & 71.4 \\
  & English -- Prop     & 43.7 & 86.6  & 50.4 & 35.1 & 86.6  & 40.5 & 59.6 & 86.6  & 68.8 \\
  & English -- Comb     & \mbf{51.5}$^{\dagger}$ & \mbf{99.1}  & \mbf{51.9} & --  & --   & --  & -- & -- & -- \\
  & French -- Direct    & 65.2 & \mbf{99.8}  & 65.3 & 32.2 & 100.0 & 32.2 & 81.8 & 100.0 & 81.8 \\
  & French -- Prop      & 48.9 & 76.7  & 63.8 & 39.7 & 76.7  & 51.7 & 64.6 & 76.7  & 84.3 \\
  & French -- Comb      & \mbf{66.3} & \mbf{99.8}  & \mbf{66.4} & --  & --   & --  & -- & -- & -- \\
  & German -- Direct    & 59.2 & 99.6  & 59.4 & 22.9 & 100.0 & 22.9 & 76.5 & 100.0 & 76.5 \\
  & German -- Prop      & 49.2 & 85.1  & 57.8 & 46.3 & 85.1  & 54.5 & 68.7 & 85.1  & 80.8 \\
  & German -- Comb      & \mbf{59.4} & \mbf{99.7}  & \mbf{59.5} & --  & --   & --  & -- & -- & -- \\
  & Spanish -- Direct   & \mbf{46.5} & \mbf{99.6}  & 46.7 & 24.8 & 100.0 & 24.8 & 65.0 & 100.0 & 65.0 \\
  & Spanish -- Prop     & 38.5 & 79.2  & \mbf{48.6} & 27.4 & 79.2  & 34.6 & 58.9 & 79.2  & 74.3 \\
  & Spanish -- Comb     & \mbf{46.5} & \mbf{99.6}  & 46.7 & --  & --   & --  & -- & -- & -- \\
  & Portuguese -- Direct & 59.3 & \mbf{99.4} & 59.7 & 32.8 & 100.0 & 32.8 & 77.6 & 100.0 & 77.6 \\
  & Portuguese -- Prop   & 56.4 & 87.7 & 64.3 & 43.2 & 87.7  & 49.2 & 71.9 & 87.7  & 82.0 \\
  & Portuguese -- Comb  & \mbf{63.2}$^{\dagger}$ & \mbf{99.4}  & \mbf{63.5} & --  & --   & --  & -- & -- & -- \\
  & Italian -- Direct   & 63.9 & 99.4  & 64.3 & 36.2 & 100.0 & 36.2 & 81.8 & 100.0 & 81.8 \\
  & Italian -- Prop     & 54.1 & 82.2  & 65.8 & 41.8 & 82.2  & 50.8 & 69.5 & 82.2  & 84.5 \\
  & Italian -- Comb     & \mbf{67.7}$^{\dagger}$ & \mbf{99.6}  & \mbf{68.0} & --  & --   & --  & -- & -- & -- \\
\cmidrule(lr){2-11}
  & Macro-avg -- Direct & 57.3 & \mbf{99.5} & 57.6 & 30.8 & \mbf{100.0} & 30.8 & \mbf{75.7} & \mbf{100.0} & 75.7 \\
  & Macro-avg -- Prop   & 48.5       & 82.9 &  58.5 & \mbf{38.9} & 82.9  & \mbf{46.9} & 65.5 & 82.9  & \mbf{79.1} \\
  & Macro-avg -- Comb   & \mbf{59.1}$^{\dagger}$ & \mbf{99.5} & \mbf{59.3} & -- & -- & -- & -- & -- & -- \\
\midrule
\multirow{2}{*}{FewRel}
  & Direct & 48.7 & \mbf{100.0} & 48.7 & \mbf{42.2} & \mbf{100.0} & 42.2 & \mbf{67.7} & \mbf{100.0} & 67.7 \\
  & Prop   & 40.2 & 75.8  & \mbf{53.1} & 40.3 & 75.8  & \mbf{53.2} & 51.5 & 75.8  & \mbf{68.0} \\
  & Comb  & \mbf{50.0}$^{\dagger}$ & \mbf{100.0} & 50.0 & -- & -- & -- & -- & -- & -- \\
\bottomrule
\end{tabular}
\caption{Results on the SMiLER and FewRel benchmarks. $\dagger$ indicates a statistically significant improvement of Comb over Direct (bootstrap test, $p < 0.05$). Atomization significantly improves relation recall. But entity recall decreases. The accuracy increases for the small LLM (Qwen3-0.6B). Bold indicates the best result per extractor column for each benchmark/language. }
\label{tab:results_smiler_fewrel}
\end{table*}

\subsection{Baselines \& Configurations}
We compare the performance of direct pipeline (\texttt{GLiREL} or \texttt{Qwen3}) vs. \texttt{MPropositionneur-V2} + direct pipeline. Our hypothesis is that atomization reduces the syntactic noise that typically hinders relation extraction on complex sentences.

For the evaluation, we analyse three different configurations:

 \begin{itemize}
     \item \textbf{Direct}:  Triplets are directly extracted from source text.
     \item  \textbf{Prop}: Triplets are extracted only from the atoms produced by the propositioner.
     \item \textbf{Comb}: For SMiLER and FewRel benchmarks, since the entity oracle is known, if entities are found, the associated triples are saved; otherwise, the triples are extracted from atomic propositions. Comb is therefore a fallback method.
     \item \textbf{Union}: The triplets are extracted from atomic propositions and from the raw paragraph/document independently. Triplets from both the \textbf{direct} and \textbf{prop} pipelines are merged. The objective is to evaluate the contribution of atomic propositions to the direct pipeline.
 \end{itemize}

The baseline approach is a triplet extractor using \texttt{GLiREL}.
Then we compare the baseline with prompted approaches using LLMs, where models are queried to generate the triplets from either the source text (\textbf{direct}) or from propositions (\textbf{prop}). 
For prompt models, we selected \texttt{Qwen3-0.6B} and \texttt{Qwen3-4B} instruct models; these models are comparable in terms of size (number of weights) to the size of the propositioner model. 
It should be noted that these models were not fine-tuned for the triplet extraction task. Rather, a zero-shot instruction approach was employed (see the prompt used in Figure~\ref{fig:PromptTemplateStage2}).
 
\section{Results and Analysis}
\label{sec:results}

In this section we report and discuss the results of the designed experiments. We evaluate quantitatively the propositioner for the triplet extraction method. 
By flattening the text, relations are made explicit, allowing the extractors to capture facts that would otherwise be missed in complex sentences.

\subsection{Evaluation on Multilingual Triplet Extraction} 
We report in Table \ref{tab:results_smiler_fewrel} results in accuracy (number of triplet correctly extracted), entity-recall and relation-recall on both SMiLER and FewRel benchmarks compared to \texttt{GLiREL} approach. 

We also report results for the different configurations: "Direct", where models try to extract the triplet directly from the raw text;  "Prop", where the models are only considering the set of propositions to extract relations; "Comb", which is the combination of raw text and the list of atomic propositions.

First, looking at the Macro-avg, we can observe that in almost all cases, the number of correctly extracted relations benefits from the atomic propositions. While direct pipelines achieve better accuracy and entity recall, a combination of raw text and atomic propositions yields better performance for small models, demonstrating that both methods support the retrieval of different entities or relations. 
This combination of the two approaches has a positive impact on the proposition in the triplet extraction pipelines. 

However, within this pipeline we consider that we know a relation exists between two entities, and thus it corresponds to relations classification setting rather than a triplet extraction setting. From a model perspective, we show that a larger LLM-based approach is more powerful across all metrics (\texttt{Qwen3-4B}), although it comes at a higher computational and memory cost.

\paragraph{Evaluation on triplet extraction on documents.} Table \ref{tab:results_docred} presents results that favour proposition granularity for GLiREL, which achieves a higher F1 score when using the propositioner.
However larger models (\texttt{Qwen3-4B}) reach higher scores (precision, recall and F1). Thus, for larger documents, larger models are preferable to the propositioner.  We hypothesize that this difference in score could be reduced by considering the constructed graph and leveraging relation transitivity.
Deeper experiments should be conducted to verify this assumption.


\paragraph{Evaluation on triplet extraction on openIE context.} As shown in Table \ref{tab:results_carb}, we observe that using the propositioner upstream of CoreNLP (Prop method) improves recall and AUC, but not precision.
This drop in precision can be explained by the increased number of triplets extracted by the atomic propositions method. We observe that in OpenIE context, the union of propositions and original text is optimal for recall and AUC, having higher Recall and AUC for all pipelines.  

\begin{table*}[!ht]
\centering
\small
\begin{tabular}{@{}lccccccccc@{}}
\toprule
& \multicolumn{3}{c}{GLiREL} & \multicolumn{3}{c}{Qwen3-0.6B} & \multicolumn{3}{c}{Qwen3-4B} \\
\cmidrule(lr){2-4}\cmidrule(lr){5-7}\cmidrule(lr){8-10}
Pipeline & P & R & F1 & P & R & F1 & P & R & F1 \\
\midrule
Sentence & 3.7 & 9.4  & 5.3  & 22.6 & 8.3  & 12.1 & \mbf{36.6} & 13.3 & 19.5 \\
Document & 1.4 & 14.8 & 2.5  & 20.3 & \mbf{15.7} & \mbf{17.7} & 32.0 & \mbf{24.4} & \mbf{27.7} \\
Proposition & \mbf{5.9} & 6.9  & \mbf{6.3}  & \mbf{25.9} & 6.7  & 10.6 & 36.5 & 9.4  & 14.9 \\
Sentence+Proposition   & 3.8 & 12.3 & 5.8 & -- & -- & -- & -- & -- & -- \\
Document+Proposition   & 1.6 & \mbf{18.1} & 2.9       & -- & -- & -- & -- & -- & -- \\
\bottomrule
\end{tabular}
\caption{Results on the DocRED benchmark. Atomization significantly improves precision and F1 when using a weaker extractor. However, when using LLM extraction, performance decreases.}
\label{tab:results_docred}
\end{table*}

\subsection{Qualitative Results} 
\label{sec:qualitative}

Even with no errors in the pipeline components, it remains possible that some transitive relations present in the input text are absent from the knowledge graph built from the triplets. 

Due to the atomic splitting of information during intermediary representation translation, such relations may become latent. However, it is important to note that these propositions can nevertheless be recovered through the use of inference, which is supported by the formal properties of the atomic propositions.

\begin{table*}[!ht]
\centering
\small
\begin{tabular}{@{}lcccccccccccc@{}}
\toprule
& \multicolumn{4}{c}{CoreNLP} & \multicolumn{4}{c}{Qwen3-0.6B} & \multicolumn{4}{c}{Qwen3-4B} \\
\cmidrule(lr){2-5}\cmidrule(lr){6-9}\cmidrule(lr){10-13}
Pipeline & P & R & F1 & AUC & P & R & F1 & AUC & P & R & F1 & AUC \\
\midrule
Direct & \mbf{17.6} & 24.3 & 20.4 & 0.144 & \mbf{29.4} & 7.9  & 12.4 & 0.051 & \mbf{46.2} & 33.3 & \mbf{38.7} & 0.244 \\
Prop   & 16.4 & \mbf{31.3} & \mbf{21.5} & \mbf{0.182} & 24.0 & 11.1 & 15.2 & 0.069 & 31.6 & 35.0 & 33.2 & 0.230 \\
Union & -- & -- & -- & -- & 23.7 & \mbf{14.7} & \mbf{18.1} & \mbf{0.091} & 30.8 & \mbf{43.6} &
  36.1 & \mbf{0.285} \\
\bottomrule
\end{tabular}
\caption{Results on the CaRB benchmark. Atomization significantly improves recall and F1 when a weaker extractor is used. However, when using LLM extraction, performance decreases.}
\label{tab:results_carb}
\end{table*}

For instance, when the following sentence is entered: 

"\texttt{Šafov is a village and municipality (obec) in Znojmo District in the South Moravian Region of the Czech Republic.}",

the propositioner's output is as follows:

\texttt{[..., "Šafov is located in Znojmo District.", "Znojmo District is located in the South Moravian Region.", "The South Moravian Region is located in the Czech Republic.", ...]}\\

The corresponding graph, displayed in Figure~\ref{img:kg}, illustrates the aforementioned transitive relation. Precisely, while triplet (``Šafov'', ``hasLocation'', ``Czech Republic'') is not directly extracted from propositions,  we could deduce it by taking into account transitivity.

\begin{figure}[h!]
    \centering
    \includegraphics[width=0.40\textwidth]{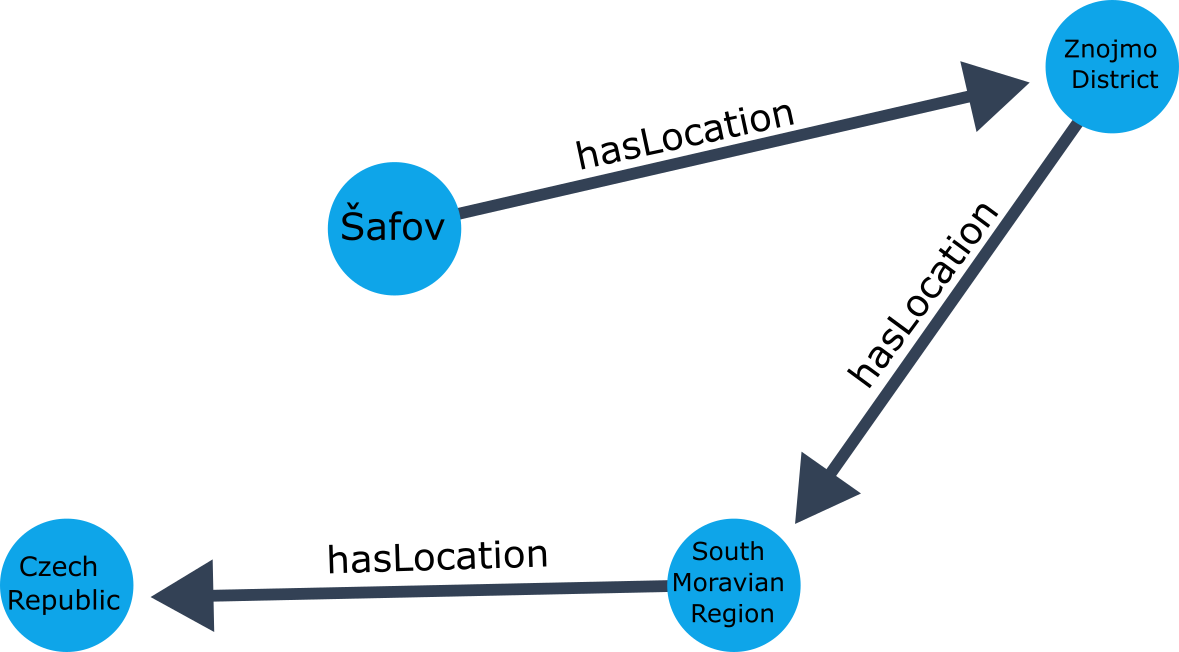}
    \caption{The Knowledge Graph built with triplets extracted by \texttt{GLiREL} on the atomic propositions for the above sentence.}
    \label{img:kg}
\end{figure}

\section{Conclusion}
\label{sec:conclusion}
In this work, we empirically demonstrate the benefits of the atomic proposition for triplet entity relation extraction. In particular, we show that decomposing documents or paragraphs into atoms helps retrieve the relation efficiently, improving performance on both the FewRel and SMiLER benchmarks. 

In addition, we observe that the combination of original input text with atomised text yields better performance across all metrics. The results obtained on both the CaRB and DocRED benchmarks validate the previous hypothesis, showing better recall performance.

However, future studies should be conducted to strengthen the methods, especially a human evaluation of the propositioner, to ensure that text units are indeed all atomic according to the definition (see Section~\ref{sec:formalism}). 

Additionally, we think that such triplet extraction methods and the constructed graph could be used in applications such as information retrieval systems.

We think that the interpretability of such text units (atoms) could provide a statistical basis for creating or deducing a logical rule, which could benefit the KG traversal algorithm. These research leads are left for further work.

\section{Limitations} 
One limitation of this study is the relevance of the recursive propositioner. To date, we have not compared propositions obtained with and without recursive refinement. The decision to use such an algorithm to refine atoms recursively was based on preliminary experiments in which we observed that some propositions were not atomic. In future works, we plan to compare the recursive propositioner to the non-recursive one.

Another limitation is the absence of evaluation using larger LLMs than \texttt{Qwen3-4B} for fair comparison and computational efficiency. Nonetheless, it is worth noticing that a larger model could be compared in future studies.

\clearpage
\section{Bibliographical References}
\bibliographystyle{lrec2026-natbib}
\bibliography{refs}


\appendix

\section{Prompts}
\label{sec:prompts}

\subsection{Evaluation Prompts}

\paragraph{Closed IE (SMiLER, FewRel, DocRED).}
The following prompt is used to classify the relation between two identified entities:

\begin{verbatim}
Given the text, identify the relation 
between the two entities.

Text: {text}
Entity 1: {e1}
Entity 2: {e2}

Choose exactly one relation from 
this list: {labels}

Answer with just the relation name, 
nothing else.
\end{verbatim}

\paragraph{Open IE (CaRB).}
The following prompt is used for open-domain triplet extraction:

\begin{verbatim}
Extract all factual 
(subject, predicate, object)
triples from the sentence.
One triple per line in the format:
subject | predicate | object
No explanations. If no triple can be 
extracted, write nothing.

Sentence: {text}
\end{verbatim}

\subsection{Propositioner Training Prompt}
\label{sec:prompt_propositioner}

The following prompt is used during the knowledge distillation training of \texttt{MPropositionneur-V2}. The teacher model (\texttt{Qwen3-32B}) is prompted to decompose a Wikipedia passage into fully autonomous atomic propositions, which are then used as targets for the student model (\texttt{Qwen3-0.6B}).

\begin{verbatim}
You are an expert in disambiguation 
and information extraction.
You must decompose the text into 
atomic propositions (single facts) 
that are FULLY AUTONOMOUS.

ABSOLUTE RULES:
1. ZERO PRONOUNS: "He", "She", "They", 
   "His", "Her", "Its", "This one" 
   ARE FORBIDDEN.
   ALWAYS replace them with the 
   full name of the entity.
2. CONTEXT: Each sentence must be 
   readable alone without knowing 
   its source.
3. REPETITION: Repeat the subject 
   in EACH sentence.

OUTPUT FORMAT: Only a JSON array of 
strings.

Title: {title}
Content: {content}
Output:
\end{verbatim}

\section{Proofs for the Formal Grounding}
\label{sec:proofs}

\begin{definition}[Safe Cut]
    A formula's cut $\phi$ in a formula $\psi$ is \textbf{safe} if $\psi$ is a sub-formula of $\phi$ and if $I(\phi) > I(\psi)$
\end{definition}

\begin{definition}[Bad cut]
    A cut of a formula $\phi$ in a formula $\psi$ is \textbf{bad} if $\psi$ is a sub-formula of $\phi$ and if $I(\phi) \leq I(\psi)$
\end{definition}

\begin{lemma}[Divisibility of Conjunction — Lemma~\label{lem:conj}]                                                                   
Let $\phi = A \land B$ where $A$ and $B$ are logically independent. Extracting the component $A$ is a \textbf{strictly safe} operation.
\end{lemma}

\begin{proof}
By definition of conjunction, $\cont(A \land B) = \cont(A) \cap \cont(B)$.
Because $A$ and $B$ are independent, $\cont(A \land B) \subsetneq \cont(A)$.
By monotonicity of $\mu$, we have $\mu(\cont(A \land B)) < \mu(\cont(A))$.
The function $-\log_2$ is strictly decreasing, so:
\[ I(A \land B) > I(A). \]
The information content of $A$ is strictly less than that of $A \land B$ : the operation is strictly safe.
\end{proof}

\begin{lemma}[Indivisibility of Disjunction — Lemma~\label{lem:disj}]
Let $\phi = A \lor B$ where $A$ and $B$ are logically independent. Extracting the component $A$ is a \textbf{bad} operation.
\end{lemma}

\begin{proof}
By definition of disjunction, $\cont(A \lor B) = \cont(A) \cup \cont(B)$.
We have the strict inclusion $\cont(A) \subsetneq \cont(A \lor B)$, hence $\mu(\cont(A)) < \mu(\cont(A \lor B))$, which yields:
\[ I(A) > I(A \lor B). \]
The information content of $A$ is strictly greater than that of $A \lor B$: decomposing a disjunction hallucinates a more specific fact than what was originally stated.
\end{proof}

\paragraph{Case of implication~\label{lem:impl}.}
Cutting $A \to B$ into $A$ is bad, because $A \to B \equiv \lnot A \lor B$, and Lemma~\ref{lem:disj} applies directly by substitution.

\begin{theorem}[Structural characterisation — Theorem\label{thm:atomicite}]
A formula $\phi$ is \textbf{atomic} if and only if it is logically equivalent to a \textbf{clause} (finite disjunction of literals).
\end{theorem}

\begin{proof}
Let $\phi$ be in Conjunctive Normal Form (CNF) :
\[\phi \equiv C_1 \land C_2 \land \dots \land C_n,\]
where each $C_i$ is a clause (disjunction of literals).

\textit{($\Rightarrow$)} Let $\phi$ be atomic. \textit{Ad absurdum}, let $n \geq 2$. By Lemma~\ref{lem:conj}, the operation $\phi \mapsto C_1$ is strictly safe, which contradicts atomicity. Therefore, $n = 1$: $\phi$ is a single clause.

\textit{($\Leftarrow$)} Let $\phi = L_1 \lor \dots \lor L_k$. By Lemma~\ref{lem:disj} (generalised by induction on $k$), all strict sub-formulae have greater information content than $\phi$. Therefore, every cut is bad, and $\phi$ atomic.
\end{proof}

\end{document}